\documentclass{article}

\usepackage{amssymb}
\usepackage{amsmath}
\usepackage{amsthm}

\newtheorem{lemma}{Lemma}

\newtheorem{theorem}{Theorem}

\renewcommand{\eqref}[1]{Equation~(\ref{#1})}

\newcommand{\figref}[1]{Figure~\ref{#1}}
\newcommand{\secref}[1]{Section~\ref{#1}}

\newcommand{\thmref}[1]{Theorem~\ref{#1}}
\newcommand{\lemref}[1]{Lemma~\ref{#1}}

\newcommand{\algref}[1]{Algorithm~\ref{#1}}

\newcommand{\cM}{\mathcal{M}}

\newcommand{\cN}{\mathcal{N}}

\newcommand{\bE}{\mathbb{E}}
\newcommand{\reals}{\mathbb{R}}

\newcommand{\tr}{\mathrm{tr}}

\usepackage{times}
\usepackage{hyperref}
\usepackage{url}
\usepackage{algorithm}
\usepackage{algorithmic}
\usepackage{eqparbox}
\usepackage{subcaption}
\usepackage{graphicx}
\usepackage{float}

\usepackage{tikz}
\usepackage{pgfplots} 
\pgfplotsset{compat=newest} 
\pgfplotsset{plot coordinates/math parser=false} 
\newlength\figureheight 
\newlength\figurewidth

\newcommand{\tx}{\tilde{x}}
\newcommand{\tw}{\tilde{w}}
\newcommand{\tell}{\tilde{\ell}}
\newcommand{\tL}{\tilde{L}}
\newcommand{\tu}{\tilde{u}}
\newcommand{\tU}{\tilde{U}}
\newcommand{\tv}{\tilde{v}}
\newcommand{\tV}{\tilde{V}}
\newcommand{\tW}{\tilde{W}}
\newcommand{\tsig}{\tilde{\sigma}}
\newcommand{\tSig}{\tilde{\Sigma}}
\newcommand{\tX}{\tilde{X}}
\newcommand{\tA}{\tilde{A}}
\newcommand{\tZ}{\tilde{Z}}
\newcommand{\tkappa}{\tilde{\kappa}}
\newcommand{\hkappa}{\hat{\kappa}}
\newcommand{\R}{\mathbb{R}}

\author{Alon Gonen\footnote{School of Computer Science, The Hebrew University, Jerusalem, Israel} \and Francesco Orabona\footnote{Yahoo Labs, New York, NY, USA} \and Shai Shalev-Shwartz\footnote{School of Computer Science, The Hebrew University, Jerusalem, Israel}}

\title{Solving Ridge Regression using\\Sketched Preconditioned SVRG}

\begin{document}

\maketitle

\begin{abstract} 
We develop a novel preconditioning method for ridge regression, based on recent linear sketching methods. By equipping Stochastic Variance Reduced Gradient (SVRG) with this preconditioning process, we obtain a significant speed-up relative to fast stochastic methods such as SVRG, SDCA and SAG.
\end{abstract}

\section{Introduction}  \label{sec:intro}
Consider the \emph{ridge regression} problem:
\begin{equation}
\label{eq:ridge}
\min_{w \in \reals^d}\  \left\{L(w) = \frac{1}{n}\sum_{i=1}^n \frac{1}{2}(w^\top x_i-y_i)^2 + \frac{\lambda}{2} \|w\|^2\right\},
\end{equation}
where $\lambda>0$ is a regularization parameter, $x_i\in\R^d$ and $y_i\in \R$ for $i=1,\cdots,n$ the training data. We focus on the large scale regime, where both $n$ and $d$ are large. In this setting, stochastic iterative methods such as SDCA~\cite{shalev2013stochastic}, SVRG~\cite{johnson2013accelerating}, and SAG~\cite{roux2012stochastic} have become a standard choice for minimizing the objective $L$. Specifically, the overall complexity of a recent improved variant of SVRG due to \cite{xiao2014proximal} depends on the average condition number, which is defined as follows. Denote the empirical correlation matrix and its eigenvalue decomposition by
\begin{equation} \label{eq:covSpec}
C := \frac{1}{n} \sum_{i=1} ^n x_i x_i^\top = \sum_{i=1} ^d \lambda_i u_i u_i^\top~.
\end{equation}
The average condition number of $C+\lambda I$ is defined as the ratio between the trace of the Hessian of $L$ and its minimal eigenvalue:
\begin{equation} \label{eq:initialCondition}
\hkappa :=\hkappa(C+\lambda I) = \frac{\tr(C+\lambda I)}{\lambda_d(C+\lambda I)} = \sum_{i=1}^d \frac{\lambda_i+\lambda}{\lambda_d+\lambda}  ~.
\end{equation}
The mentioned variant of SVRG finds an $\epsilon$-approximate minimizer of $L$ in time $\tilde{O}((\hkappa+n)d \log(1/\epsilon))$. Namely, the output of the algorithm, denoted $\hat{w}$, satisfies $\bE[L(\hat{w})]-L(w^\star) \le \epsilon$, where the expectation is over the randomness of the algorithm. 
For an accelerated version of the algorithm, we can replace $\hkappa$ by $\sqrt{n \hkappa}$~\cite{shalev2014accelerated,lin2015universal}. 

The regularization parameter, $\lambda$, increases the smallest eigenvalue of $C+\lambda I$ to be at least $\lambda$, thus improves the condition number and makes the optimization problem easier.
However, to control the under/over fitting tradeoff, $\lambda$ has to decrease as $n$ increases~\cite{shalev2014understanding}.
Moreover, in many machine learning applications $\lambda_d$ approaches zero and it is usually smaller than the value of $\lambda$.
Overall, this yields a large condition number in most of the interesting cases.

A well-known approach for reducing the average condition number is \emph{preconditioning}. Concretely, for a (symmetric) positive definite (pd) matrix $P \in \reals^{d \times d}$, we define the preconditioned optimization problem as 
\begin{equation}
\label{eq:condRidge}
\min_{\tw \in \reals^d} \tL(\tw):= L(P^{-1/2} \tw) ~.
\end{equation}
Note that $\tw$ is an $\epsilon$-approximate minimizer of $\tL$ if and only if $w = P^{-1/2} \tw$ forms an $\epsilon$-approximate minimizer of $L$. Hence, we can minimize \eqref{eq:condRidge} rather than \eqref{eq:ridge}. As we shall see, the structure of the objective allows us to apply the preconditioning directly to the data (as a preprocessing step) and consequently rewrite the preconditioned objective as a ridge regression problem with respect to the preconditioned data (see \secref{sec:regCond}). For a suitable choice of a matrix $P$, the average condition number is significantly reduced. Precisely, as will be apparent from the analysis,
the pd matrix that minimizes the average condition number is $P=C+\lambda I$, and the corresponding average condition number is $d$. However, we note that such preconditioning process would require both the computation of $P^{-1/2}$ and the computation of $P^{-1/2} x_i$ for each $i \in [n]$. By first order conditions, computing $(C+\lambda I)^{-1/2}$ is equivalent to solving the original problem in \eqref{eq:ridge}, rendering this ``optimal'' preconditioner useless.

Yet, the optimal preconditioner might not needed in many cases.
In fact, a common empirical observation (see \secref{sec:empirical}) is that (high-dimensional) machine learning problems tend to have few dominant features, while the other coordinates are strongly correlated with the stronger features. As a result, the spectrum of the correlation matrix decays very fast. Hence, it is natural to expect to gain a lot from devising preconditioning methods that focus on the stronger directions of the data.

Our contributions are as follows. We develop a relatively cheap preconditioning method that, coupled with SVRG,
assures to speed-up the convergence in practical applications while having a computational cost comparable to SVRG alone.
In order to approximately extract the stronger directions while incurring a low computational cost, we rely on a variant of the Block Lanczos method due to \cite{musco2015stronger} in order to compute an approximated truncated SVD (Singular Value Decomposition) of the correlation matrix $C$. Finally, by equipping SVRG with this preconditioner, we obtain our main result.

\section{Main Result} \label{sec:mainResult}
\begin{theorem} \label{thm:main}
Let $k \in [d]$ be a given parameter and assume that the regularization parameter, $\lambda$, is larger than $\lambda_d$. Our preconditioning process runs in time $O(ndk \log(n))$.
By equipping the SVRG of \cite{xiao2014proximal} with this preconditioner, we find an $\epsilon$-approximate minimizer for \eqref{eq:ridge} (with probability at least $9/10$) in additional runtime of $O ((\tilde{\kappa}  + n + d )d \log(1/\epsilon) )$, where $\tilde{\kappa}=\frac{k \lambda_k+ \sum_{i>k} \lambda_i}{\lambda}$ or $\tilde{\kappa} = \left(\frac{n(k \lambda_k+ \sum_{i>k} \lambda_i)}{\lambda}\right)^{1/2}$ if we use accelerated SVRG.
\end{theorem}
When the runtimes of both the (accelerated) SVRG and our preconditioned (accelerated) SVRG are controlled by the average condition number (and both runtimes dominate $ndk$), then ignoring logarithmic dependencies, we obtain a speed-up of order 
\begin{equation} \label{eq:ratioRidge}
\textrm{ratio} = \frac{\sum_{i=1}^d \lambda_i}{k\,\lambda_k+\sum_{i>k}\lambda_i}
= \frac{\sum_{i=1}^k \lambda_i +\sum_{i>k}\lambda_i }{k\,\lambda_k ~~~~+\sum_{i>k}\lambda_i}~.
\end{equation} 
(or $\sqrt{\sum_{i=1}^d \lambda_i/(\lambda_k k+\sum_{i>k}\lambda_i)}$ if acceleration is used) over SVRG.  If the spectrum decays fast then $k \, \lambda_k \ll \sum_{i=1}^k \lambda_i$ and $\sum_{i>k}\lambda_i \ll k \, \lambda_k$. In this case, the ratio will be large.  Indeed, as we show in the experimental section, this ratio is often \emph{huge} for relatively small $k$.

\subsection{Main challenges and perspective}
While the idea of developing a preconditioner that focuses on the stronger directions of the data matrix sounds plausible, there are several difficulties that have to be solved.
\begin{itemize}
\item
First, since a preconditioner must correspond to an invertible transformation, it is not clear how to form a preconditioner based on a low rank approximation and, in particular, how should we treat the non-leading components.
\item
One of the main technical challenges in our work is to translate the approximation guarantees of the Lanczos method into a guarantee on the resulted average condition number. The standard measures of success for low-rank approximation are based on either Frobenius norm or spectral norm errors. As will be apparent from the analysis (see \secref{sec:lowCond}), such bounds do not suffice for our needs. 
Our analysis relies on stronger per vector error guarantees \eqref{eq:perVec} due to \cite{musco2015stronger}.
\end{itemize}
It should be emphasized that while we use a variant of SVRG due to \cite{xiao2014proximal}, we could equally use a variant of SDCA \cite{shalev2016sdca} or develop such a variant for SAG or SAGA. Furthermore, while we focus on the quadratic case, we believe that our ideas can be lifted to more general setting. For example, when applied to  self-concordant functions, each step of Newton's method requires the minimization of a quadratic objective. Therefore, it is natural to ask if we can benefit from applying our method for approximating the Newton step. 

\subsection{Bias-complexity tradeoff} \label{sec:reduceReg}
As we mentioned above, $\lambda$ controls a tradeoff between underfitting and overfitting. In this view, we can interpret our result as follows. Assuming for simplicity that $n \ge d$ and ignoring logarithmic dependencies, we note that if
\begin{equation} \label{eq:decay}
\lambda = \frac{k \lambda_k+\sum_{i>k}\lambda_i}{nk}~,
\end{equation}
then the runtime of our preconditioned SVRG is $\tilde{O}(ndk)$. For comparison, the runtime of (unconditioned) SVRG is $\tilde{O}(ndk)$ if
\begin{equation} \label{eq:decaySVRG}
\lambda = \frac{\sum_{i=1}^d \lambda_i}{nk}~.
\end{equation}
The ratio between the RHS of \eqref{eq:decaySVRG} and \eqref{eq:decay} is the ratio given in \eqref{eq:ratioRidge}. Hence, for a given ``runtime budget'' of order $\tilde{O}(ndk)$, we can set the regularization parameter of the preconditioned SVRG to be smaller by this ratio. Similar interpretation holds for the accelerated versions.

\section{Related Work} \label{sec:related}
\paragraph{Existing algorithms and their complexities: }
Since minimizing \eqref{eq:ridge} is equivalent to solving the system $(C+\lambda I)w=\frac{1}{n} \sum_{i=1}^n y_i x_i$, standard numerical linear algebra solvers such as Gaussian elimination can be used to solve the problem in time $O(nd^2)$. 

Iterative deterministic methods, such as Gradient Descent (GD), finds an $\epsilon$-approximate minimizer in time $nd \kappa \log(1/\epsilon)$, where $\kappa=\frac{\lambda_1(C+\lambda I)}{\lambda_d(C+\lambda I)}$ is the condition number of $C+\lambda I$ (see Theorem 2.1.15 in \cite{nesterov2004introductory}). The Kaczmarz algorithm~\cite{kaczmarz1937angenaherte} has an identical complexity. Both the Conjugate Gradient (CG) method~\cite{hestenes1952methods} and the Accelerated Gradient Descent (AGD) algorithm of \cite{nesterov1983method} enjoy a better runtime of $nd \sqrt{\kappa}\log(1/\epsilon)$. In fact, CG has a more delicate analysis (see Corollary 16.7 in \cite{vishnoi2012laplacian}): If all but $c \in [d]$ eigenvalues of $C+\lambda I$ are contained in a range $[a,b]$, then the runtime of CG is at most $nd(c+\sqrt{b/a} \log(1/\epsilon))$. In particular, CG's runtime is at most $O(nd^2)$. Furthermore, following the interpretation of our main result in \secref{sec:reduceReg}, we note that for a ``runtime budget'' of $\tilde{O}(ndk)$, we can set the regularization parameter of CG to be of order $\lambda_k/k^2$ (which is usually much greater than the RHS of \eqref{eq:decay}).


\paragraph{Linear Sketching: }
Several recently developed methods in numerical linear algebra are based on the so-called \emph{sketch-and-solve} approach, which essentially suggests that given a matrix $A$, we first replace it with a smaller random matrix $AS$, and then perform the computation on $AS$ \cite{woodruff2014sketching,clarkson2013low,sarlos2006improved}. For example, it is known that if the entries of $S$ are i.i.d. standard normal variables and $S$ has $p=\Omega(k/\epsilon)$ columns, then with high probability, the column space of $AS$ contains a $(1+\epsilon)$ rank-$k$ approximation to $A$ with respect to the Frobenius norm. This immediately yields a fast PCA algorithm (see Section 4.1 in \cite{woodruff2014sketching}). 

While the above sketch-and-solve approach sounds promising for this purpose, our analysis reveals that controlling the Frobenius norm error does not suffice for our needs. 
We need spectral norm bounds, which are known to be more challenging~\cite{witten2013randomized}. 
Furthermore, as mentioned above, the success of our conditioning method heavily depends on the stronger per vector error guarantees \eqref{eq:perVec} obtained by \cite{musco2015stronger} which are not obtained by simpler linear sketching methods. 

\paragraph{Sketched preconditioning: }
Recently, subspace embedding methods were used to develop cheap preconditioners for linear regression with respect to the squared loss~\cite{woodruff2014sketching}. Precisely, \cite{clarkson2013low} considered the case $\lambda=0$ (i.e, standard least-squares) and developed a preconditioning method that reduces the average condition number to a constant. Thereafter, they suggest applying a basic solver such as CG. The overall running time is dominated by the preconditioning process which runs in time $\tilde{O}(d^3+nd)$. Hence, a significant improvement over standard solvers is obtained if $n \gg d$.

The main shortcoming of this method is that it does not scale well to large dimensions. Indeed, when $d$ is very large, the overhead resulted from the preconditioning process can not be afforded. 

\paragraph{Efficient preconditioning based on random sampling: } While we focus on reducing the dependence on the dimensionality of the data, other work  investigated the gain from using only a random subset of the data points to form the conditioner~\cite{yang2014data}. The theoretical gain of this approach has been established under coherence assumptions~\cite{yang2014data}.

\section{Preliminaries} \label{sec:pre}
\subsection{Additional notation and definitions} \label{sec:notation}
Any matrix $B \in \reals^{d \times n}$ of rank $r$ can be written in (thin) SVD form as $B = U \Sigma V^\top = \sum_{i=1}^r \sigma_i(B) u_i v_i^\top$. The singular values are ordered in descending order. The spectral norm of $B$ is defined by $\|B\|=\sigma_1(B)$. The spectral norm is submultiplicative, i.e., $\|AB\| \le \|A\|\|B\|$ for all $A$ and $B$. Furthermore, the spectral norm is unitary invariant, i.e., for all $A$ and $U$ such that the columns of $U$ are orthonormal, $\|UA\| = \|A\|$. For any $k \in [r]$, it is well known that the truncated SVD of $B$, $B_k := U_k \Sigma_k V_k =  \sum_{i=1} ^{k} \sigma_i(B) u_i v_i^\top$, is the best rank-$k$ approximation of $B$ w.r.t. the spectral norm~\cite{trefethen1997numerical}. 
A twice continuously differentiable function $f: \reals^d \rightarrow \reals$ is said to be $\beta$-smooth if  $\|\nabla^2 f(w)\| \leq \beta$ for all $w$, where $\nabla^2 f(w)$ is the Hessian of $f$ at $w$. $f$ is said to be $\alpha$-strongly convex if $\lambda_d(\nabla^2 f(w)) \ge \alpha$ for all $w$. If $g$ is convex and $f$ is $\alpha$-strongly convex, then $f+g$ is $\alpha$-strongly convex.

\subsection{Stochastic Variance Reduced Gradient (SVRG)} \label{sec:svrg}
We consider a variant of the Stochastic Variance Reduced Gradient (SVRG) algorithm of \cite{johnson2013accelerating} due to \cite{xiao2014proximal}. The algorithm is an epoch-based iterative method for minimizing an average, $F(w) = \frac{1}{N} \sum_{i=1}^N f_i(w)$, of smooth functions. It is assumed that each $f_i:\reals^d \rightarrow \reals$ is convex and $\beta_i$-smooth. The entire function $F$ is assumed to be $\alpha$-strongly convex. The algorithm is detailed in \algref{alg:svrg}. 
Its convergence rate depends on the averaged smoothness of the individual functions and the average condition number of $F$, defined as 
\begin{equation} \label{eq:conditionSVRG}
\hat{\beta} = \frac{1}{N} \sum_{i=1}^N \beta_i ~~~;~~~ \hkappa_F = \frac{\hat{\beta}}{\alpha} ~.
\end{equation}
\begin{theorem} \textbf{\cite{xiao2014proximal}} \label{thm:svrg}
Fix $\epsilon>0$. Running SVRG (\algref{alg:svrg}) with any $w_0$, $S \ge \log((F(w_0) - \min_{w \in \reals^d} F(w))/\epsilon)$, $m = \lceil \hkappa_F \rceil$, and  $\eta = 0.1/\hat{\beta}$ yields an $\epsilon$-approximate minimizer of $F$. Furthermore, assuming that  each single gradient $\nabla f_i(w)$ can be computed in time $O(d)$, the overall runtime is $O((\hkappa_F+N)d \log(\epsilon_0/\epsilon))$. 
\end{theorem}
\begin{algorithm}[t]
\caption{SVRG cite{xiao2014proximal}}
\label{alg:svrg}
\begin{algorithmic}[1]
\STATE \textbf{Input:} Functions $f_1,\ldots,f_n, \beta_1,\ldots,\beta_n$
\STATE \textbf{Parameters:} $\bar{w}_0 \in \reals^d$, $m$, $\eta$, $S \in \mathbb{N}$
\FOR {$s=1,2,\ldots,S$}
\STATE $\bar{w} = \bar{w}_{s-1}$  
\STATE $\bar{v} = \nabla F(\bar{w})$
\STATE $w_0 = \bar{w}$ 
\FOR[New epoch] {$t=1,\ldots,m$}   
\STATE Pick $i_t \in [N]$ with probability $q_{i_t}=\beta_{i_t}/\sum \beta_j$
\STATE $v_t = (\nabla f_{i_t}(w_{t-1}) -\nabla f_{i_t}(\bar{w}))/q_{i_t}+\bar{v}$
\STATE $w_t = w_{t-1} - \eta v_t$
\ENDFOR 
\STATE 
$\bar{w}_s = \frac{1}{m} \sum_{t=1}^m w_t$
\ENDFOR
\STATE \textbf{Output:} the vector $\bar{w}_S$
\end{algorithmic}
\end{algorithm}
In the original definition of SVRG~\cite{johnson2013accelerating}, the indices $i_t$ are chosen uniformly at random from $[n]$, rather than proportional to $\beta_i$. As a result, the convergence rate depends on the maximal smoothness, $\max \{\beta_i\}$, rather than the average, $\hat{\beta}$. It will be apparent from our analysis (see \thmref{thm:condEffect}) that in our case, $\max \{\beta_i\}$ is proportional to the maximum norm of any preconditioned $x_i$. Since we rely on the improved variant of \cite{xiao2014proximal}, our bound depends on the average of the $\beta_i$'s, which scale with the average norm of the preconditioned $x_i$'s. To simplify the presentation, in the sequel we refer to \algref{alg:svrg} as SVRG.

\subsection{Randomized Block Lanczos}
A randomized variant of the Block Lanczos method due to \cite{musco2015stronger} is detailed\footnote{More precisely, Algorithm 2 in \cite{musco2015stronger} returns the projection matrix $\tU_k \tU_k^\top$, while we also compute the SVD of $\tU_k \tU_k^\top A$. The additional runtime is negligible.} in \algref{alg:musco}. Note that the matrix $\tU_k \tSig_k \tV_k^\top$ forms an SVD of the matrix $\tA_k := Q(Q^\top A)_k = \tU_k \tU_k^\top A$.
\begin{algorithm}[t]
\caption{Block Lanczos method \cite{musco2015stronger}}
\label{alg:musco}
\begin{algorithmic}[1]
\STATE \textbf{Input: } $A \in \reals^{d \times n}, k \le d, \epsilon' \in (0,1)$
\STATE $q= \Theta \left(\frac{\log(n)}{\sqrt{\epsilon}} \right)$, $p=qk$, $\Pi \sim \cN(0,1)^{n \times k}$
\STATE Compute $K = [A \Pi,(AA^\top) A \Pi,\ldots, (AA^\top)^{q-1} A \Pi]$
\STATE Orthonormalize $K$'s columns to obtain $Q \in \reals^{d \times qk}$ 
\STATE Compute the truncated SVD $(Q^\top A)_k = \tW_k \tSig_k \tV_k^\top$   
\STATE Compute $\tU_k = Q \tW_k$ 
\STATE \textbf{Output:} the matrices $\tU_k,\tSig_k,\tV_k$ 
\end{algorithmic}
\end{algorithm}
\begin{theorem} \label{thm:musco} \textbf{\cite{musco2015stronger}}
Consider the run of \algref{alg:musco} and denote $\tA_k = \tU_k \tSig_k \tV_k = \sum_{i=1}^k \tsig_i \tu_i \tv_i^\top$. Denote the SVD of $A$ by $A = \sum_{i=1}^d \sigma_i v_i u_i^\top$. The following bounds hold with probability at least $9/10$:
\[
\|A - \tA_k\| \le (1+\epsilon') \|A-A_k\| \le (1+\epsilon') \sigma_k
\]
\begin{align} 
\forall i \in [k],~~|z_i^\top AA^\top z_i - u_i^\top AA^\top u_i| &=  |\tsig_i^2-\sigma_i^2 | \notag \\
&\le \epsilon' \sigma_{k+1}^2 ~. \label{eq:perVec}
\end{align}
The runtime of the algorithm is $O\left(\frac{nd k \log(n)}{ \sqrt{\epsilon'}}+ \frac{k^2(n+d)}{\epsilon'} \right)$.
\end{theorem}

\section{Sketched Conditioned SVRG} \label{sec:scsvrg}
In this section we develop our sketched conditioning method. By analyzing the properties of this conditioner and combining it with SVRG, we will conclude \thmref{thm:main}. 

Recall that we aim at devising cheaper preconditioners that lead to a significant reduction of the condition number.
Specifically, given a parameter $k \in [d]$, we will consider only preconditioners $P^{-1/2}$ for which both the computation of $P^{-1/2}$ itself and the computation of the set $\{P^{-1/2}x_i, \ldots, P^{-1/2} x_n\}$ can be carried out in time $\tilde{O}(ndk)$. We will soon elaborate more on the considerations when choosing the preconditioner, but first we would like to address some important implementation issues.   
\subsection{Preconditioned regularization} \label{sec:regCond}
In order to implement the preconditioning scheme suggested above, we should be able to find a simple form for the function $\tilde{L}$. In particular, since we would like to use SVRG, we should write $\tilde{L}$ as an average of $n$ components whose gradients can be easily computed. Denote by $\tx_i = P^{-1/2}x_i$ for all $i \in [n]$. Since for every $i \in [n]$, $((P^{-1/2} w)^\top x_i-y_i)^2=(w^\top \tx_i  - y_i)^2$, it seems natural to write $\tilde{L}(w)=L(P^{-1/2}w)$ as follows:
\[
\tilde{L}(w) =  \frac{1}{n} \sum_{i=1}^n \underbrace{\frac{1}{2}(w^\top \tx_i-y_i)^2}_{=:\tilde{\ell}_i}  +\frac{\lambda}{2} \|P^{-1/2}w\|^2~.
\]
Assume momentarily that $\lambda=0$. Note that the gradient of $\tilde{\ell}_i$ at any point $w$ is given by $\nabla \tilde{\ell}_i(w_t)=(w^\top \tx_i-y_i)\tx_i$. Hence, by computing all the $\tx_i$'s in advance, we are able to apply SVRG directly to the preconditioned function and computing the stochastic gradients in time $O(d)$.

When $\lambda>0$, the computation of the gradient at some point $w$ involves the computation of $P^{-1} w$. We would like to avoid this overhead. To this end, we decompose the regularization function as follows. Denote the standard basis of $\reals^d$ by $e_1,\ldots,e_d$. Note that the function $L$ can be rewritten as follows:
\begin{align*} 
L(w) = \frac{1}{n+d} \sum_{i=1}^{n+d} \ell_i(w)~,
\end{align*}
where $\ell_i(w) = \frac{n+d}{n} \frac{1}{2} (w^\top x_i-y_i)^2$ for $i=1,\ldots,n$ and $\ell_{n+i}(w)=\lambda(n+d) \frac{1}{2}(w^\top e_i)^2$ for $i=1,\ldots,d$.
Finally, denoting $b_i = P^{-1/2} e_i$ for all $i$, we can rewrite the preconditioned function $\tilde{L}$ as follows:
\begin{align*} 
\tL(w) = \frac{1}{n+d} \sum_{i=1}^{n+d} \tell_i(w)~,
\end{align*}
where $\tell_i(w) = \frac{n+d}{n} \frac{1}{2} (w^\top \tx_i-y_i)^2$ for $i=1,\ldots,n$ and $\tell_{n+i}(w)=\lambda(n+d) \frac{1}{2}(w^\top b_i)^2$ for $i=1,\ldots,d$.
By computing the $\tx_i$'s and the $b_i$'s in advance, we are able to apply SVRG while computing stochastic gradients in time $O(d)$. 
\subsection{The effect of conditioning}
We are now in position to address the following fundamental question: How does the choice of the preconditioner, $P^{-1/2}$, affects the resulted average condition number of the function $\tL$ (\ref{eq:conditionSVRG})? The following lemma upper bounds $\hkappa_{\tL}$ by the average condition number of the matrix $P^{-1/2}(C+\lambda I)P^{-1/2}$, which we denote by $\tkappa$  (when the identity of the matrix $P$ is understood). 
\begin{theorem} \label{thm:condEffect}
Let $P^{-1/2}$ be a preconditioner. Then, the average condition number of $\tilde{L}$ is upper bounded by
\[
\hkappa_{\tilde{L}} \le \tkappa = \frac{\tr (P^{-1/2} (C + \lambda I ) P^{-1/2} ) } {\lambda_d(P^{-1/2} (C + \lambda I ) P^{-1/2})} ~.
\]
\end{theorem}
The proof is in the appendix. 
Note that an optimal bound of $O(d)$ is attained by the whitening matrix $P^{-1/2}  = (C+\lambda I)^{-1/2}$.

\subsection{Exact sketched conditioning}
Our sketched preconditioner is based on a random approximation of the best rank-$k$ approximation of the data matrix. It will be instructive to consider first a preconditioner that is based on an exact rank-$k$ approximation of the data matrix. 
Let $X \in \reals^{d \times n}$ be the matrix whose $i$-th columns is $x_i$ and let $\bar{X}=n^{-1/2} X$. Denote by $\bar{X}= \sum_{i=1}^{\textrm{rank}(\bar{X})} \sigma_i u_i v_i^\top = U \Sigma V^\top$ the SVD of $\bar{X}$ and recall that $\bar{X}_k= \sum_{i=1}^k \sigma_i u_i v_i^\top$ is the best $k$-rank approximation of $\bar{X}$. Note that $\bar{X}\bar{X}^\top=C$ and therefore $\sigma_i^2 = \lambda_i(C)=\lambda_i$. Furthermore, the left singular vectors of $\bar{X}$, $u_1,\ldots,u_k$, coincide with the $k$ leading eigenvectors of the matrix $C$. Consider the preconditioner, 
$$
P^{-1/2}  = \sum_{i=1}^k \frac{u_i u_i^\top}{\sqrt{\lambda_i+\lambda}} + \frac{I-\sum_{i=1}^k  u_i u_i^\top}{\sqrt{\lambda_k+\lambda}}~,
$$
where $u_{k+1},\ldots, u_d$ are obtained from a completion of $u_1,\ldots,u_k$ to an orthonormal basis.
\begin{lemma} \label{lem:exactCondSketch}
Let $k \in [d]$ be a parameter and assume that the regularization parameter, $\lambda$, is larger than $\lambda_d$.
Using the exact sketched preconditioner, we obtain
\begin{equation} \label{eq:exactCondSketchCond}
\hkappa_{\tilde{L}} \le \frac{k \lambda_k+ \sum_{i>k} \lambda_i}{\lambda} + d ~.
\end{equation}
\end{lemma}
\begin{proof} 
A simple calculation shows that for $i=1,\ldots,k$, 
\[
\lambda_{i} (P^{-1/2} (C + \lambda I ) P^{-1/2} )  = \frac{\lambda_i+\lambda}{\lambda_i+\lambda} = 1 ~.
\]
Similarly, for $i=k+1,\ldots, d$,
\[
\lambda_{i} (P^{-1/2} (C + \lambda I ) P^{-1/2} ) = \frac{\lambda_i+\lambda}{\lambda_k+\lambda} ~.
\]
Finally,
\[
\lambda_{d} (P^{-1/2} (C + \lambda I ) P^{-1/2}) \ge \frac{\lambda}{\lambda_k+\lambda}~.
\]
Combining the above with \thmref{thm:condEffect}, we obtain that
\begin{align*} 
\hkappa_{\tilde{L}} &\le \frac{\tr(P^{-1/2} (C + \lambda I ) P^{-1/2})}{\lambda_d(P^{-1/2} (C + \lambda I ) P^{-1/2})} \\
&\le  k\frac{\lambda_k+\lambda}{\lambda} + \sum_{i=k+1}^d \frac{\lambda_i+\lambda}{\lambda}  \\
&= \frac{k \lambda_k + \sum_{i>k} \lambda_i}{\lambda} + d ~. \qedhere
\end{align*}
\end{proof}

\subsection{Sketched conditioning} \label{sec:lowCond}
An exact computation of the SVD of the matrix $\bar{X}$ takes $O(nd^2)$. Instead, we will use the Block Lanczos method in order to approximate the truncated SVD of $\bar{X}$. Specifically, given a parameter $k \in [d]$, we invoke the Block Lanczos method with the parameters $\bar{X},k$ and $\epsilon'=1/2$. Recall that the output has the form $\tilde{X}_k = \tU _k \tSig_k \tV_k^\top= \sum_{i=1}^k \tsig_i\tu_i \tv_i^\top$. Analogously to the exact sketched preconditioner, we define our sketched preconditioner by
\begin{equation} \label{eq:condSketch}
P^{-1/2}  = \sum_{i=1}^k \frac{\tu_i \tu_i^\top}{\sqrt{\tsig_i^2+\lambda}}  + \frac{I-\sum_{i=1}^k  \tu_i \tu_i^\top}{\sqrt{\tsig_k^2+\lambda}} ~.
\end{equation}
\begin{theorem}  \label{thm:sketchCond}
Let $k \in [d]$ be a parameter and assume that the regularization parameter, $\lambda$, is larger that $\lambda_d$. Using the sketched preconditioner defined in \eqref{eq:condSketch}, up to a multiplicative constant, we obtain the bound \eqref{eq:exactCondSketchCond} on the average condition number with probability at least $9/10$.
\end{theorem}
The rest of this section is devoted to the proof of \thmref{thm:sketchCond}. We follow along the lines of the proof of \lemref{lem:exactCondSketch}. Up to a multiplicative constant, we derive the same upper and lower bounds on the eigenvalues of $P^{-1/2} (C+\lambda I)P^{-1/2}$. 

From now on, we assume that the bounds in \thmref{thm:musco} (where $\epsilon'=1/2$) hold. This assumption will be valid with probability of at least $9/10$. We next introduce some notation. We can rewrite $P^{-1/2}  = \tU (\tSig^2+\lambda I)^{-1/2} \tU^\top$ where $\tSig$ is a diagonal $d \times d$ with $\tSig_{i,i} = \tsig_i$ if $i \le k$ and $\tSig_i = \tsig_k$ if $i>k$.
and the columns of $\tU$ are a completion of  $\tu_1,\ldots,\tu_k$ to an orthonormal basis. Recall that the SVD of $\bar{X}$ is denoted by $\bar{X} = \sum_{i=1}^d \sigma_i u_i v_i^\top = U \Sigma V^\top$. 
\begin{lemma} \textbf{(Upper bound on the leading eigenvalue)} We have
\[
\lambda_1(P^{-1/2} (C+\lambda I)P^{-1/2})  \le 17~.
\] 
\end{lemma}
\begin{proof}
Since $\lambda_1(P^{-1/2} (C+\lambda I)P^{-1/2}) = \|P^{-1/2} (C+\lambda I)P^{-1/2}\| = \|P^{-1/2} C P^{-1/2} + \lambda P^{-1}\|$, using the triangle inequality we have that
\[
\lambda_1(P^{-1/2} (C+\lambda I)P^{-1/2}) ~\le~ \|P^{-1/2} C P^{-1/2}\| + \lambda \|P^{-1}\| ~.
\]
By the definition of $P$ we have that $\|P^{-1}\| = \frac{1}{\tsig_k^2 + \lambda}$ and therefore the second summand on the right hand side of the above is at most $\frac{\lambda}{\tsig_k^2+\lambda} \le 1$. As to the first summand, recall that $C = \bar{X} \bar{X}^\top$ and therefore $\|P^{-1/2} C P^{-1/2}\| = \|\bar{X}^\top P^{-1/2} \|^2$. We will show that $\|\bar{X}^\top P^{-1/2} \| \le 4$ which will imply that $\|P^{-1/2} C P^{-1/2}\| \le 16$. 
To do so, we first apply the triangle inequality,
\begin{align*}
&\|\bar{X}^\top P^{-1/2}\|  = \| (\tX_k + (\bar{X} - \tX_k))^\top P^{-1/2} \| \\
&\quad \le \| \tX_k^\top P^{-1/2} \| + \|(\bar{X} - \tX_k)^\top P^{-1/2} \| ~.
\end{align*}
Let us consider one term at the time. Recall that $\tX_k = \tU _k \tSig_k \tV_k^\top$. Since $\tU_k^\top \tU  \in \reals^{k,d}$ is a diagonal matrix with ones on the diagonal, and since the spectral norm is invariant to multiplication by unitary matrices, we obtain that
\begin{align*}
&\|\tX_k^\top P^{-1/2}\| = \|\tV_k \tSig_k \tU_k^\top \tU (\tSig^2+\lambda I)^{-1/2} \tU^\top\| \\
&\quad = \|\tSig_k \tU_k^\top \tU (\tSig^2+\lambda I)^{-1/2} \| \\
&\quad = \max_{i \in [k]} \frac{\tsig_i}{ \sqrt{\tsig_i^2+\lambda}} 
\le \max_{i \in [k]} \frac{\tsig_i}{\tsig_i+\sqrt{\lambda}} 
\le 1 ~.
\end{align*}
Next, by the submutiplicativity of the spectral norm,
\[
\|(\bar{X} - \tX_k)^\top P^{-1/2} \| \le \|\bar{X}-\tX_k\| \cdot \|P^{-1/2}\| ~.
\]
\thmref{thm:musco} implies that $\|\bar{X}-\tX_k\| \le \tfrac{3}{2} \sigma_k$ and
\begin{align*}
&\|P^{-1/2}\| = \frac{1}{ \sqrt{\tsig_k^2+\lambda}} 
\le \frac{1}{\sqrt{\tsig_k^2}} \le \frac{1}{ \sqrt{\sigma_k^2 - (1/2) \sigma_{k+1}^2}} \\
&\quad \le \frac{1}{ \sigma_{k} \sqrt{\tfrac{1}{2}} } = \frac{ \sqrt{2}}{ \sigma_k} 
<\frac{ 2}{\sigma_k}~.
\end{align*}
Hence, $\|\bar{X}-\tX_k\| \cdot \|P^{-1/2}\| \le 3$. 
Combining all of the above bounds concludes our proof. 
\end{proof}
\begin{lemma} \textbf{(Refined upper bound on the last $d-k$ eigenvalues)} For any $i \in \{k+1,\ldots,d\}$,
\[
\lambda_i \left(P^{-1/2} (C+ \lambda I)P^{-1/2} \right) \le \frac{2 (\lambda_i+\lambda)}{\lambda_k+\lambda}~.
\]
\end{lemma}
\begin{proof}
Using the Courant minimax principle~\cite{bhatia2013matrix}, we obtain the following bound for all $i \in \{k+1,\ldots,d\}$:
\begin{align*}
&\lambda_i \left(P^{-1/2} (C+ \lambda I ) P^{-1/2} \right) \\
&= \max_{\substack{\mathcal{M} \subseteq \reals^d:\\ \dim(\mathcal{M}) =i}}\min_{\substack{x \in \cM:\\x \neq 0}} \frac{  x^\top P^{-1/2} (C+\lambda I) P^{-1/2} x}{\|x\|^2} \\
& = \max_{\substack{\mathcal{M} \subseteq \reals^d:\\ \dim(\mathcal{M}) =i}}\min_{\substack{x \in \cM:\\ x \neq 0}} \frac{  x^\top P^{-1/2} (C+\lambda I) P^{-1/2} x}{\|P^{-1/2}x\|^2} \cdot \frac{\|P^{-1/2}x\|^2} {\|x\|^2} \\
& \le \left(\max_{\substack{\mathcal{M} \subseteq \reals^d:\\ \dim(\mathcal{M}) =i}}\min_{\substack{x \in \cM: \\x \neq 0}} \frac{  x^\top P^{-1/2} (C+\lambda I) P^{-1/2} x}{\|P^{-1/2}x\|^2}\right) \times\\
& \qquad \left( \max_{\substack{x \in \reals^d:\\ x \neq 0}} \frac{x^\top P^{-1} x}{\|x\|^2} \right)\\
&= \lambda_i \left(C+\lambda  I \right) \cdot \lambda_1(P^{-1})  = (\lambda_i+\lambda) \cdot (\tilde{\sigma}_k^2+\lambda)^{-1} ~.
\end{align*}
Finally, using  \thmref{thm:musco} we have that $\tsig_k^2 \ge \sigma_k^2 - \frac{1}{2} \sigma_{k+1}^2 \ge \tfrac{1}{2} \sigma_k^2 = \tfrac{1}{2} \lambda_k$ and therefore,
\[
(\tilde{\sigma}_k^2+\lambda)^{-1} ~\le~ (\tfrac{1}{2} \lambda_k +\lambda)^{-1} ~\le~ 2\, (\lambda_k+\lambda)^{-1} ~. \qedhere
\]
\end{proof}
\begin{lemma} \label{lem:lowerSmall} \textbf{(Lower bound on the smallest eigenvalue)}
\[
\lambda_d( P^{-1/2} C P^{-1/2}) \ge \frac{\lambda}{19 (\lambda_k+\lambda)}~.
\]
\end{lemma}
\begin{proof}
Note that
\begin{equation} \label{eq:digging0}
\lambda_d (P^{-1/2} (C+\lambda I) P^{-1/2}) = \frac{1}{\|P^{1/2} (C+\lambda I)^{-1} P^{1/2}\|}~,
\end{equation}
so we can derive an upper bound on $\|P^{1/2} (C+\lambda I)^{-1} P^{1/2}\|$.
Consider an arbitrary completion of $\tilde{v}_1,\ldots,\tilde{v}_k$ to an orthonormal set, $\tilde{v}_1,\ldots,\tilde{v}_d \in \reals^n$. Let $\tilde{V} \in \reals^{n \times d}$ be the matrix whose $i$-th column is $\tv_i$. 
Since the spectral norm is unitary invariant and both $\tU$ and $\tV$ have orthonormal columns,
\begin{align*} 
&\|P^{1/2} (C+\lambda I)^{-1} P^{1/2}\| \notag\\
&= \|\tU (\tSig^2+\lambda I)^{1/2} \tU^\top (C+\lambda I)^{-1} \tU   (\tSig^2+\lambda I)^{1/2} \tU^\top\| \notag \\
&= \|\tV   (\tSig^2+\lambda I)^{1/2} \tU^\top (C+\lambda I)^{-1} \tU   (\tSig^2+\lambda I)^{1/2} \tV^\top\| ~.
\end{align*}
Denote by $\tZ = \tU   (\tSig^2+\lambda I)^{1/2} \tV^\top$. By the triangle inequality and the submutiplicativity of the spectral norm,
\begin{align}
&\|\tZ^\top(C+\lambda I)^{-1} \tZ\| \le \|\bar{X}^\top (C+\lambda I)^{-1} \bar{X}\| \notag\\
&\qquad + \|(\tZ-\bar{X})^\top (C+\lambda I)^{-1} (\tZ-\bar{X})\| \notag \\
& \le \|\bar{X}^\top (C+\lambda I)^{-1} \bar{X}\|+  \|\tZ-\bar{X}\|^2 \|(C+\lambda I)^{-1}\|~. \label{eq:digging15}
\end{align}
To bound the first summand of \eqref{eq:digging15}, we use the unitary invariance to obtain
\begin{align*} 
&\|\bar{X}^\top (C+\lambda I)^{-1} \bar{X} \| = \|V \Sigma U^\top U (\Sigma^2 +\lambda I)^{-1} U^\top U \Sigma V^\top\| \notag\\
&\quad = \|\Sigma  (\Sigma^2 +\lambda I)^{-1} \Sigma \| = \max_i \frac{\lambda_i^2}{\lambda_i^2+\lambda} \le 1~.
\end{align*}
For the second summand of \eqref{eq:digging15}, note that $\|(C+ \lambda I)^{-1}\|= \frac{1}{\lambda_d+\lambda}$ and that, using the triangle inequality,
\begin{align*}
&\|\tZ - \bar{X}\| = \|(\tU  \tSig \tV^\top - \bar{X}) + (\tZ -  \tU  \tSig \tV^\top) \|\\
&\quad \le \|\tU \tSig \tV^\top - \bar{X}\| + \|\tU ((\tSig^2+\lambda I)^{1/2}- \tSig) \tV^\top\|~.
\end{align*}
By using unitary invariance together with the inequality $\sqrt{\tsig_i^2+\lambda} - \tsig_i \le \sqrt{\lambda}$ (which holds for every $i$), we get
\[
\|\tU ((\tSig^2+\lambda I)^{1/2}- \tSig) \tV^\top\| = \|(\tSig^2+\lambda I)^{1/2}- \tSig\| \le \sqrt{\lambda}~.
\]
Hence, using the inequality $(x+y)^2 \le 2x^2 + 2y^2$, we obtain
\begin{align*}  
\|\tZ - \bar{X}\|^2 &\le  2 \|\tU \tSig \tV^\top - \bar{X}\|^2 + 2\lambda ~.
\end{align*}
We next derive an upper bound on $\|\tU \tSig \tV^\top - \bar{X}\|$. Since $\tU \tSig \tV^\top = \tilde{X}_k + \tsig_k \sum_{i=k+1}^d \tu_i \tv_i^\top$, 
\[
\|\tU \tSig \tV^\top-\bar{X}\| \le \|\tilde{X}_k - \bar{X}\|+\tsig_k \left\|\sum_{i=k+1}^d \tu_i \tv_i^\top\right\| ~.
\]
Using \thmref{thm:musco} we know that $\|\tilde{X}_k - \bar{X}\| \le 1.5 \, \sigma_k$ and that $\tsig_k \le \sqrt{\sigma_k^2+0.5 \, \sigma_{k+1}^2} \le  1.5 \, \sigma_k$. Combining this with the fact that $\|\sum_{i=k+1}^d \tu_i \tv_i^\top\| = 1$, we obtain
\begin{align*} 
\|\tU \tSig \tV^\top-\bar{X}\| \le 3\,\sigma_k ~.
\end{align*} 
Combining the above inequalities, we obtain
\begin{align*}
\|P^{1/2} (C+\lambda I)^{-1} P^{1/2}\| 
&\le  1 + \frac{2 \cdot(3 \sigma_k)^2+2\lambda}{\lambda_d+\lambda} \\
&\le \frac{19(\lambda_k+\lambda)}{\lambda} ~,
\end{align*}
and using \eqref{eq:digging0} we conclude our proof.
\end{proof}
\begin{proof} \textbf{(of \thmref{thm:sketchCond})}
The three last lemmas imply that the inequalities derived during the proof of \lemref{lem:exactCondSketch} remain intact up to a multiplicative constant. Therefore, the bound \eqref{eq:exactCondSketchCond} on the condition number also holds up to a multiplicative constant. This completes the proof.
\end{proof}

\subsection{Sketched Preconditioned SVRG}
By equipping SVRG with the sketched preconditioner \eqref{eq:condSketch}, we obtain the Sketched Preconditioned SVRG (see \algref{alg:SCSVRG}). 
\begin{proof} \textbf{(of \thmref{thm:main})}
The theorem follows from \thmref{thm:sketchCond} and \thmref{thm:svrg}.
\end{proof}
\begin{algorithm}
\caption{Sketched Preconditioned SVRG}
\label{alg:SCSVRG}
\begin{algorithmic}[1]
\STATE \textbf{Input: } $x_1,\ldots, x_n \in \reals^d, y_1,\ldots,y_n \in \reals, \epsilon>0$
\STATE \textbf{Parameters: } $\lambda>0, k \in [d]$
\STATE Let $\bar{X} \in \reals^{d,n}$ be the matrix whose $i$'th column is $(1/n) x_i$
\STATE Run the Block Lanczos method (\algref{alg:musco}) with the input $\bar{X}, k, \epsilon'=1/2$ to obtain $\tX_k = \tU_k \tSig_k \tV_k$
\STATE Let $\tu_i$ be the columns of $\tU_k$ and $\tilde{\sigma}_i$ be the diagonal elements of $\tSig_k$
\STATE Form the preconditioner $P^{-1/2}$ according to \eqref{eq:condSketch}
\STATE Compute $\tx_i = P^{-1/2} x_i, b_i = P^{-1/2} e_i$
\STATE Let $\ell_i(w) = \frac{n+d}{n} \frac{1}{2} (w^\top \tx_i-y_i)^2$ for $i=1,\ldots,n$ and $\ell_i(w) = \lambda(n+d) (w^\top b_i)^2$ for $i=n+1,\ldots,n+d$
\STATE Let $\beta_i = \frac{n+d}{n} \|\tx_i\|^2$ for $i=1,\ldots,n$ and $\beta_i  = \lambda(n+d) \|b_i\|$ for $i=n+1,\ldots,n+d$. Let $\hat{\beta} = \frac{1}{n} \sum_{i=1}^{n+d} \beta_i$
\STATE Run SVRG (\algref{alg:svrg}) 
\STATE Return $\hat{w} = P^{1/2} \tilde{w}$
\end{algorithmic}
\end{algorithm}

\section{The Empirical Gain of Sketched Preconditioning} \label{sec:empirical}
\begin{figure}[t]
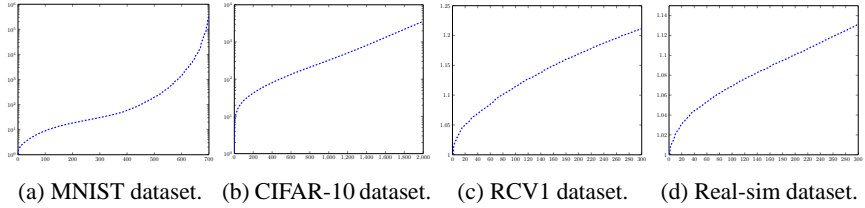

\centering
\begin{subfigure} [b] {0.23 \textwidth}
\resizebox{\linewidth}{!}{\input{MNIST_ratio.tex}}
\caption{MNIST dataset.}
\end{subfigure}
\begin{subfigure} [b] {0.23\textwidth}
\resizebox{\linewidth}{!}{\input{cifar10_ratio.tex}}
\caption{CIFAR-10 dataset.}
\end{subfigure}
\begin{subfigure} [b] {0.23\textwidth}
\resizebox{\linewidth}{!}{
%
%
\begin{tikzpicture}

\begin{axis}[%
width=4.520833in,
height=3.565625in,
at={(0.758333in,0.48125in)},
scale only axis,
separate axis lines,
every outer x axis line/.append style={black},
every x tick label/.append style={font=\color{black}},
xmin=0,
xmax=300,
every outer y axis line/.append style={black},
every y tick label/.append style={font=\color{black}},
ymin=1,
ymax=1.25
]
\addplot [color=blue,dashed,line width=2.0pt,forget plot]
  table[row sep=crcr]{%
1	1\\
2	1.01402594123662\\
3	1.01860018629882\\
4	1.02136588262184\\
5	1.02335912993182\\
6	1.02515942057121\\
7	1.02917141767784\\
8	1.02953414636905\\
9	1.03207518184767\\
10	1.03428217991716\\
11	1.03643850219783\\
12	1.03866055285582\\
13	1.04074449611849\\
14	1.04202158645743\\
15	1.04499552759349\\
16	1.0461171115014\\
17	1.04636755501581\\
18	1.04810932554616\\
19	1.04918674016273\\
20	1.05113667167068\\
21	1.05178621403767\\
22	1.05239587953982\\
23	1.05279307554445\\
24	1.05364747945582\\
25	1.05476757368529\\
26	1.05572014513903\\
27	1.05645264802599\\
28	1.05771964532249\\
29	1.05940406559988\\
30	1.06160946679535\\
31	1.06248385962575\\
32	1.06285653740383\\
33	1.0639857600549\\
34	1.06467323255181\\
35	1.06506421935771\\
36	1.06648017111551\\
37	1.06673867043446\\
38	1.06753876028363\\
39	1.06843880521581\\
40	1.06896409430248\\
41	1.06988794215114\\
42	1.07066547390866\\
43	1.07167115589193\\
44	1.07239499885376\\
45	1.07376008453848\\
46	1.07456100859082\\
47	1.07489920009349\\
48	1.07592207315668\\
49	1.07670100795669\\
50	1.07719472038748\\
51	1.07775647485864\\
52	1.07881561543016\\
53	1.07955341509848\\
54	1.08006100220406\\
55	1.08089067753624\\
56	1.08142367675284\\
57	1.08261658574159\\
58	1.08323601326628\\
59	1.08388598391255\\
60	1.08492891782974\\
61	1.08633542996438\\
62	1.08647999712967\\
63	1.08741253473438\\
64	1.08788138944511\\
65	1.08918403944365\\
66	1.09053897301767\\
67	1.09180425238302\\
68	1.09279682488631\\
69	1.09331301473933\\
70	1.0935257807559\\
71	1.09477118866675\\
72	1.0959292409004\\
73	1.09749001800698\\
74	1.09775111873198\\
75	1.09792495084587\\
76	1.09869539136021\\
77	1.09930040851564\\
78	1.09969040325005\\
79	1.10038964246283\\
80	1.10132367626794\\
81	1.10155500985492\\
82	1.10268524132743\\
83	1.10294736751206\\
84	1.10387872067552\\
85	1.1046704523866\\
86	1.10503620759385\\
87	1.10596965638843\\
88	1.10657109461039\\
89	1.10673347357743\\
90	1.10746706472307\\
91	1.10809094695978\\
92	1.1089813181921\\
93	1.10921558022693\\
94	1.11002141274028\\
95	1.11075522269021\\
96	1.11182830345364\\
97	1.11337478897513\\
98	1.11357425833062\\
99	1.11391455579121\\
100	1.11438044628769\\
101	1.11490327899682\\
102	1.11553410764791\\
103	1.11609609607663\\
104	1.1165313997269\\
105	1.11758787951312\\
106	1.11797851538485\\
107	1.11809493987834\\
108	1.11928394469437\\
109	1.12084846470872\\
110	1.1208710102232\\
111	1.12179510838522\\
112	1.12256166759686\\
113	1.12288035098625\\
114	1.12349199936638\\
115	1.12428112934862\\
116	1.12523637161812\\
117	1.12570437736535\\
118	1.12619547916688\\
119	1.12682719571364\\
120	1.12715764285277\\
121	1.12764125277459\\
122	1.12797240969747\\
123	1.12858457808231\\
124	1.12923067521188\\
125	1.12941620310765\\
126	1.12997197014814\\
127	1.13044396424695\\
128	1.1308228419169\\
129	1.13118008024592\\
130	1.13202302880744\\
131	1.13278839453459\\
132	1.13320423798878\\
133	1.13376253098062\\
134	1.13439077025219\\
135	1.13523773487727\\
136	1.13583518736892\\
137	1.13635868874253\\
138	1.13727689672674\\
139	1.13765062428436\\
140	1.13799627845762\\
141	1.13847216307225\\
142	1.13862607827138\\
143	1.13934148266185\\
144	1.1412752899646\\
145	1.1418421573297\\
146	1.14224297288443\\
147	1.14254995633449\\
148	1.14298895916342\\
149	1.14385542031764\\
150	1.14410927340094\\
151	1.14493793612653\\
152	1.14505454482597\\
153	1.14585995604759\\
154	1.14660130199495\\
155	1.14678933487358\\
156	1.14756084103879\\
157	1.14785381246449\\
158	1.14887960220837\\
159	1.14921664232705\\
160	1.14974649116063\\
161	1.14981975868198\\
162	1.1509546279005\\
163	1.15113855057573\\
164	1.15159650215541\\
165	1.15181445913573\\
166	1.15263500482951\\
167	1.15295917213176\\
168	1.15411084335107\\
169	1.15424234770183\\
170	1.1550999839757\\
171	1.15550268822966\\
172	1.15585970224915\\
173	1.15623909865325\\
174	1.15634739262133\\
175	1.15724039908751\\
176	1.15777487466236\\
177	1.15860911936911\\
178	1.15874434142339\\
179	1.15988083133052\\
180	1.16008139137519\\
181	1.16032498619525\\
182	1.16059225597856\\
183	1.16134018044655\\
184	1.16167664254599\\
185	1.16204380819557\\
186	1.1621821525602\\
187	1.1631685778076\\
188	1.16338525829916\\
189	1.16360090744322\\
190	1.16444323962093\\
191	1.16470114354453\\
192	1.16532405749346\\
193	1.16562538143605\\
194	1.16602296336221\\
195	1.16655217613467\\
196	1.16673466825057\\
197	1.16809004750205\\
198	1.16854143978116\\
199	1.16898224264739\\
200	1.16917257128062\\
201	1.16999722659004\\
202	1.17039228805625\\
203	1.17059918133842\\
204	1.17093852807641\\
205	1.17155826890628\\
206	1.17196656439773\\
207	1.17261670306833\\
208	1.1728816276526\\
209	1.1732128851101\\
210	1.17344218760187\\
211	1.1740419405325\\
212	1.17435704859861\\
213	1.17514667121496\\
214	1.17539648667607\\
215	1.17593730117322\\
216	1.17635188423373\\
217	1.17662746628917\\
218	1.17719880053644\\
219	1.17737180051713\\
220	1.17829377571469\\
221	1.17889637340737\\
222	1.17914637060878\\
223	1.1802284220164\\
224	1.18038400141518\\
225	1.18072260959043\\
226	1.18120791815603\\
227	1.18164125319649\\
228	1.1818759684627\\
229	1.18267692499699\\
230	1.18315077537718\\
231	1.18356708457362\\
232	1.18393424178982\\
233	1.18417653047176\\
234	1.18485753467523\\
235	1.18510431266411\\
236	1.18550753752815\\
237	1.18566569754676\\
238	1.18615777901684\\
239	1.18701005720136\\
240	1.18713458326963\\
241	1.18797046836492\\
242	1.18807035167666\\
243	1.18880978502948\\
244	1.18904567594385\\
245	1.18927428101349\\
246	1.1900888797356\\
247	1.19020959622425\\
248	1.19036613201472\\
249	1.19128840169807\\
250	1.19133646019734\\
251	1.19170069267273\\
252	1.19206449097626\\
253	1.19259716158079\\
254	1.19300181997385\\
255	1.1935179906621\\
256	1.19381204738234\\
257	1.19412342161471\\
258	1.19470950967615\\
259	1.19478262146931\\
260	1.19573376752855\\
261	1.19637013421703\\
262	1.19664979107844\\
263	1.19677559243842\\
264	1.19709879147346\\
265	1.1972908913096\\
266	1.19784380456999\\
267	1.19868494966468\\
268	1.19957097557015\\
269	1.19986809898797\\
270	1.20000519174864\\
271	1.20013561931903\\
272	1.20077492210147\\
273	1.20102359543593\\
274	1.20162076353498\\
275	1.20176742363714\\
276	1.20235730216598\\
277	1.20260003782618\\
278	1.20293435305855\\
279	1.20332523886745\\
280	1.20355717371458\\
281	1.20397728156774\\
282	1.20422470686168\\
283	1.2049743435993\\
284	1.20551949715247\\
285	1.20553526252147\\
286	1.20610724500622\\
287	1.20667187713318\\
288	1.20726245165937\\
289	1.20750811455006\\
290	1.20791349812692\\
291	1.20816136470336\\
292	1.208659036781\\
293	1.2087305060451\\
294	1.20906374636384\\
295	1.20985192160362\\
296	1.21007096439419\\
297	1.21053270915248\\
298	1.21074088342348\\
299	1.21096660069556\\
300	1.21148780357198\\
};
\end{axis}
\end{tikzpicture}%
}
\caption{RCV1 dataset.}
\end{subfigure}
\begin{subfigure} [b] {0.23\textwidth}
\resizebox{\linewidth}{!}{
%
%
\begin{tikzpicture}

\begin{axis}[%
width=4.520833in,
height=3.565625in,
at={(0.758333in,0.48125in)},
scale only axis,
separate axis lines,
every outer x axis line/.append style={black},
every x tick label/.append style={font=\color{black}},
xmin=0,
xmax=300,
every outer y axis line/.append style={black},
every y tick label/.append style={font=\color{black}},
ymin=1,
ymax=1.15
]
\addplot [color=blue,dashed,line width=2.0pt,forget plot]
  table[row sep=crcr]{%
1	1\\
2	1.00764363032567\\
3	1.0100298596894\\
4	1.01081994179907\\
5	1.01258567590885\\
6	1.01360862241832\\
7	1.01444154481784\\
8	1.01661669819033\\
9	1.01838330878142\\
10	1.02079958215662\\
11	1.02222200506322\\
12	1.02266742418701\\
13	1.02403673602565\\
14	1.02468147241925\\
15	1.02522871033024\\
16	1.02642469925982\\
17	1.02851403286365\\
18	1.02906064433087\\
19	1.02987786405084\\
20	1.03075404028508\\
21	1.03093204660281\\
22	1.03273224075624\\
23	1.03336243638006\\
24	1.03353451262843\\
25	1.03423781609039\\
26	1.03495171623266\\
27	1.0354477536871\\
28	1.03640763898119\\
29	1.03707363086146\\
30	1.03783908193518\\
31	1.0387103300997\\
32	1.03899110028361\\
33	1.04001530702844\\
34	1.04058787972682\\
35	1.04173347141161\\
36	1.04232074970171\\
37	1.04273503353955\\
38	1.04318831462259\\
39	1.0439664432818\\
40	1.0440962044802\\
41	1.04498067135114\\
42	1.04528177741218\\
43	1.04572123638068\\
44	1.04587839058173\\
45	1.04621675827224\\
46	1.04639120921256\\
47	1.04725942198439\\
48	1.04787040449617\\
49	1.04789470155603\\
50	1.04867093195344\\
51	1.04903068977486\\
52	1.04938770781163\\
53	1.04980149207588\\
54	1.05059804418239\\
55	1.05097576097213\\
56	1.05159287245736\\
57	1.0518952583737\\
58	1.0523895857696\\
59	1.05302036930488\\
60	1.05309828142107\\
61	1.05378686729241\\
62	1.05412132095322\\
63	1.05451980518256\\
64	1.05525743462364\\
65	1.05551123691315\\
66	1.05625405487401\\
67	1.05646076978823\\
68	1.05675560734493\\
69	1.0574617906828\\
70	1.05772662138812\\
71	1.05845584742314\\
72	1.05893985271513\\
73	1.05937037320471\\
74	1.05965093602382\\
75	1.05995170722184\\
76	1.06053301259959\\
77	1.06066594957827\\
78	1.06118210657195\\
79	1.06157651851837\\
80	1.06172919834156\\
81	1.06225061333682\\
82	1.06249525135613\\
83	1.06300079764613\\
84	1.06333275149717\\
85	1.06381097364858\\
86	1.06434228884624\\
87	1.06443103691288\\
88	1.06514276836818\\
89	1.06553128192411\\
90	1.06561160217139\\
91	1.06604031833627\\
92	1.06630265474319\\
93	1.06654143945486\\
94	1.06752083809674\\
95	1.06766430956545\\
96	1.0678435288967\\
97	1.06833054721836\\
98	1.06856127214906\\
99	1.06862477097176\\
100	1.0690556067574\\
101	1.06960555465629\\
102	1.06996332424667\\
103	1.07023481229943\\
104	1.07069143997386\\
105	1.07112146283799\\
106	1.07161576544015\\
107	1.07191144218018\\
108	1.07221384478878\\
109	1.07234541976691\\
110	1.07268359091049\\
111	1.07281824995372\\
112	1.07402701998544\\
113	1.07437539195649\\
114	1.07446005824482\\
115	1.07471681086318\\
116	1.07500131740494\\
117	1.07532682411356\\
118	1.07556329996778\\
119	1.07587624382074\\
120	1.07610793855072\\
121	1.07686976417054\\
122	1.07695850500073\\
123	1.07729736190987\\
124	1.07755306772222\\
125	1.07796389423833\\
126	1.07837072565892\\
127	1.07862462487484\\
128	1.07913668901083\\
129	1.07916852826486\\
130	1.07953507333155\\
131	1.07979760514636\\
132	1.07993685412849\\
133	1.08011719102465\\
134	1.08057557833451\\
135	1.08105686820527\\
136	1.08141677283348\\
137	1.08160195198209\\
138	1.08191640181134\\
139	1.0823697377538\\
140	1.08268039767846\\
141	1.08314228969992\\
142	1.08353905348147\\
143	1.08372097878338\\
144	1.08396081923586\\
145	1.08436609346403\\
146	1.08472809883189\\
147	1.08479556840108\\
148	1.08500015919969\\
149	1.08544329930966\\
150	1.08556958422347\\
151	1.08576429468872\\
152	1.08588983049005\\
153	1.08617377833275\\
154	1.08678597444012\\
155	1.08704269333432\\
156	1.08729481324209\\
157	1.08745845301619\\
158	1.08803924972945\\
159	1.08859108387523\\
160	1.08881468210782\\
161	1.08907575106415\\
162	1.08960664364109\\
163	1.09016412158658\\
164	1.09039526840133\\
165	1.09071923128866\\
166	1.09112801514618\\
167	1.09116643897274\\
168	1.09148716845467\\
169	1.09163076309517\\
170	1.09202175283693\\
171	1.09223603394016\\
172	1.09278859291749\\
173	1.09298030401978\\
174	1.09316938614905\\
175	1.09351084840128\\
176	1.09398746679083\\
177	1.09425309595358\\
178	1.09437144826808\\
179	1.09457466250884\\
180	1.09467694084072\\
181	1.09519565004073\\
182	1.09545391104363\\
183	1.0957075885235\\
184	1.09589366255832\\
185	1.09628423339728\\
186	1.09637636838226\\
187	1.09682218090567\\
188	1.09701859558139\\
189	1.09724592809137\\
190	1.09768246529776\\
191	1.09807984878813\\
192	1.09853275499896\\
193	1.09879315642251\\
194	1.09908961603971\\
195	1.09924252999747\\
196	1.09958757327129\\
197	1.0996055624724\\
198	1.09971047236197\\
199	1.10050731137848\\
200	1.10083484409486\\
201	1.10104450485166\\
202	1.10131688621673\\
203	1.10176856311282\\
204	1.10213330085499\\
205	1.10231955849988\\
206	1.10266092415062\\
207	1.102828864202\\
208	1.10292450277065\\
209	1.10312405120266\\
210	1.10332273198501\\
211	1.10339089666655\\
212	1.10377417865116\\
213	1.10430585371976\\
214	1.10457457045024\\
215	1.10484931909577\\
216	1.10548593617232\\
217	1.10556047321576\\
218	1.10607548096043\\
219	1.10624181010622\\
220	1.10651906076281\\
221	1.10675129647097\\
222	1.10705782012385\\
223	1.1075103659776\\
224	1.10798637271278\\
225	1.10842981727213\\
226	1.1086077137217\\
227	1.10877305650748\\
228	1.10908325550179\\
229	1.10933205348064\\
230	1.1096417500662\\
231	1.11002776840479\\
232	1.11018164058552\\
233	1.11051231094368\\
234	1.11114844375829\\
235	1.11126055277591\\
236	1.11157575707982\\
237	1.11183250035875\\
238	1.11234291463443\\
239	1.11247604573677\\
240	1.11272275656859\\
241	1.11291081033135\\
242	1.11315683863619\\
243	1.11363270384888\\
244	1.11404696088654\\
245	1.1144084523622\\
246	1.11482194386537\\
247	1.11529946593139\\
248	1.11552251668828\\
249	1.11563748020249\\
250	1.116030506359\\
251	1.11619217609509\\
252	1.11651379453544\\
253	1.11689409190341\\
254	1.11730634751028\\
255	1.11765869210397\\
256	1.11782203421496\\
257	1.11811575675947\\
258	1.1184784617473\\
259	1.11883390230149\\
260	1.11901606301368\\
261	1.11949017402441\\
262	1.11956468629072\\
263	1.11992069606789\\
264	1.12021134697054\\
265	1.12052459578278\\
266	1.12066440919583\\
267	1.12126177828212\\
268	1.12138654671953\\
269	1.12174019172351\\
270	1.12228400347148\\
271	1.12245688420446\\
272	1.12252826079051\\
273	1.12281475759836\\
274	1.1231866744213\\
275	1.12339657456568\\
276	1.12379995956283\\
277	1.12390917233097\\
278	1.12442376839572\\
279	1.12457096652883\\
280	1.12479581404996\\
281	1.12506229899095\\
282	1.12516764961205\\
283	1.12595185730551\\
284	1.12614563543863\\
285	1.12669444542658\\
286	1.12703148473383\\
287	1.12707396210777\\
288	1.12753415002105\\
289	1.12759912134388\\
290	1.12794876604827\\
291	1.12808710464556\\
292	1.12867640535501\\
293	1.1288894180263\\
294	1.12935513337134\\
295	1.12952470183352\\
296	1.12987309066823\\
297	1.13029987252362\\
298	1.13073598949885\\
299	1.13076559292285\\
300	1.13121401719107\\
};
\end{axis}
\end{tikzpicture}%
}
\caption{Real-sim dataset.}
\end{subfigure}
\caption{Plot of the ratio \eqref{eq:ratioRidge} as a function of $k$.}
\label{fig:ratio}
\end{figure}

\begin{figure}[t]
\centering
\begin{subfigure}{0.23 \textwidth}
\resizebox{\linewidth}{!}{
%
%
\begin{tikzpicture}

\begin{axis}[%
width=4.520833in,
height=3.565625in,
at={(0.758333in,0.48125in)},
scale only axis,
separate axis lines,
every outer x axis line/.append style={black},
every x tick label/.append style={font=\color{black}},
xmin=0,
xmax=30,
every outer y axis line/.append style={black},
every y tick label/.append style={font=\color{black}},
ymode=log,
ymin=0.0001,
ymax=0.1,
yminorticks=true,
legend style={legend cell align=left,align=left,draw=black}
]
\addplot [color=red,dashed,line width=2.0pt]
  table[row sep=crcr]{%
1	0.0538393434159012\\
2	0.00660045369356519\\
3	0.00356257522605144\\
4	0.00294328053112055\\
5	0.00263111779242342\\
6	0.00242552733462194\\
7	0.00227520305036721\\
8	0.00216006383514539\\
9	0.00206779945714465\\
10	0.00199141365857574\\
11	0.00192682701870223\\
12	0.00187108038315407\\
13	0.0018220902546056\\
14	0.00177867908151379\\
15	0.00173980287730807\\
16	0.0017045170239793\\
17	0.00167232230703335\\
18	0.0016427363082334\\
19	0.00161538177929311\\
20	0.00158997372935995\\
21	0.00156624021619109\\
22	0.00154399265030189\\
23	0.00152304742362335\\
24	0.00150326146908378\\
25	0.00148450564458936\\
26	0.00146670591642939\\
27	0.00144974356392886\\
28	0.00143356639793385\\
29	0.00141807921962349\\
30	0.00140323877542597\\
};
\addlegendentry{SVRG};

\addplot [color=green,solid,line width=2.0pt,mark=asterisk,mark options={solid}]
  table[row sep=crcr]{%
1	0.073217593011316\\
2	0.0105108595835415\\
3	0.00382509900771853\\
4	0.00245196935899098\\
5	0.00187398736590881\\
6	0.00154316278604149\\
7	0.00132401688058461\\
8	0.00116616848570685\\
9	0.00104508844414814\\
10	0.000948268283222643\\
11	0.000868875424448861\\
12	0.000801904159600218\\
13	0.000744334273864729\\
14	0.000694136547185975\\
15	0.000649819596683394\\
16	0.000610350051494594\\
17	0.000574799988738084\\
18	0.000542638756194903\\
19	0.000513333507702244\\
20	0.000486435992978536\\
21	0.000461615196020323\\
22	0.000438617754143109\\
23	0.000417289564377114\\
24	0.000397406327238263\\
25	0.000378814568791733\\
26	0.000361377274478812\\
27	0.000345034677303666\\
28	0.000329628679350337\\
29	0.00031512928121846\\
30	0.000301401954125571\\
};
\addlegendentry{SCSVRG-k=30};

\end{axis}
\end{tikzpicture}%
}
\caption{Synthetic with linear decay}
\label{fig:emp1_a1}
\end{subfigure}
\begin{subfigure}{0.23 \textwidth}
\resizebox{\linewidth}{!}{
%
%
\begin{tikzpicture}

\begin{axis}[%
width=4.520833in,
height=3.565625in,
at={(0.758333in,0.48125in)},
scale only axis,
separate axis lines,
every outer x axis line/.append style={black},
every x tick label/.append style={font=\color{black}},
xmin=0,
xmax=20,
every outer y axis line/.append style={black},
every y tick label/.append style={font=\color{black}},
ymode=log,
ymin=1e-07,
ymax=0.1,
yminorticks=true,
legend style={legend cell align=left,align=left,draw=black}
]
\addplot [color=red,dashed,line width=2.0pt]
  table[row sep=crcr]{%
1	0.0613350791047182\\
2	0.00258992793774234\\
3	0.000215870503118625\\
4	9.44003570623559e-05\\
5	7.4342003280394e-05\\
6	6.43801509720328e-05\\
7	5.79671251181172e-05\\
8	5.34381564703292e-05\\
9	5.00758542056152e-05\\
10	4.74494161810954e-05\\
11	4.53650226210408e-05\\
12	4.36711818655759e-05\\
13	4.22535375144096e-05\\
14	4.1032219601625e-05\\
15	3.99736889904112e-05\\
16	3.90375466624831e-05\\
17	3.81970663414733e-05\\
18	3.74379027231796e-05\\
19	3.67445426901359e-05\\
20	3.61057959653308e-05\\
};
\addlegendentry{SVRG};

\addplot [color=green,solid,line width=2.0pt,mark=asterisk,mark options={solid}]
  table[row sep=crcr]{%
1	0.0702959354199167\\
2	0.00351665505772703\\
3	0.000307919752233233\\
4	9.12207593745393e-05\\
5	4.68840642005881e-05\\
6	2.79466622814372e-05\\
7	1.77093742513403e-05\\
8	1.1644566034758e-05\\
9	7.88274134751361e-06\\
10	5.45615931730331e-06\\
11	3.83884128844075e-06\\
12	2.73967771742024e-06\\
13	1.98320376360289e-06\\
14	1.45447174601087e-06\\
15	1.08317544816276e-06\\
16	8.19798640993609e-07\\
17	6.31193395333671e-07\\
18	4.96541989163622e-07\\
19	3.99492264137324e-07\\
20	3.29280688139699e-07\\
};
\addlegendentry{SCSVRG-k=30};

\end{axis}
\end{tikzpicture}%
}
\caption{Synthetic with quadratic decay}
\label{fig:emp1_a2}
\end{subfigure}
\begin{subfigure}{0.23 \textwidth}
\resizebox{\linewidth}{!}{
%
%
\begin{tikzpicture}

\begin{axis}[%
width=4.520833in,
height=3.565625in,
at={(0.758333in,0.48125in)},
scale only axis,
separate axis lines,
every outer x axis line/.append style={black},
every x tick label/.append style={font=\color{black}},
xmin=0,
xmax=60,
every outer y axis line/.append style={black},
every y tick label/.append style={font=\color{black}},
ymode=log,
ymin=0.0001,
ymax=0.1,
yminorticks=true,
legend style={legend cell align=left,align=left,draw=black}
]
\addplot [color=red,dashed,line width=2.0pt]
  table[row sep=crcr]{%
1	0.0150789611775456\\
2	0.00354206677063262\\
3	0.00214557995725254\\
4	0.00167613931146035\\
5	0.00144245849252583\\
6	0.00129583239118017\\
7	0.00119348823075545\\
8	0.00111559396253519\\
9	0.00105340884817912\\
10	0.00100144123938174\\
11	0.000957023911171362\\
12	0.000918005159840043\\
13	0.000883425359825885\\
14	0.000852282026910262\\
15	0.000824123163060492\\
16	0.000798394209163528\\
17	0.000774826820521074\\
18	0.000753089011537464\\
19	0.000732876438174893\\
20	0.000714140125028309\\
21	0.000696553432434202\\
22	0.000680097332748583\\
23	0.000664632819700348\\
24	0.000650053833001846\\
25	0.000636300327884684\\
26	0.00062327100073134\\
27	0.000610910431453421\\
28	0.000599176822405303\\
29	0.000587993512878304\\
30	0.000577354364103985\\
31	0.000567177564091442\\
32	0.000557451809463316\\
33	0.000548111139653987\\
34	0.000539169020752475\\
35	0.000530595992352417\\
36	0.000522338322880317\\
37	0.000514395442765048\\
38	0.00050672680945274\\
39	0.000499345705041537\\
40	0.000492201903129152\\
41	0.000485309907866716\\
42	0.000478634269049066\\
43	0.000472169693750801\\
44	0.00046590752494928\\
45	0.00045983647129913\\
46	0.000453948989657935\\
47	0.000448227313625052\\
48	0.000442664537700782\\
49	0.000437263739641203\\
50	0.000432008696658062\\
51	0.000426887235559933\\
52	0.000421911962491736\\
53	0.00041706542347604\\
54	0.000412328003331006\\
55	0.000407707678777763\\
56	0.00040320286936453\\
57	0.000398810265568345\\
58	0.000394512311052253\\
59	0.000390313011430957\\
60	0.000386211926312593\\
};
\addlegendentry{SVRG};

\addplot [color=green,solid,line width=2.0pt,mark=asterisk,mark options={solid}]
  table[row sep=crcr]{%
1	0.0257921601567196\\
2	0.00334665905955987\\
3	0.00126752127816147\\
4	0.000892932730369181\\
5	0.000740376851195684\\
6	0.000648963767591215\\
7	0.000584339707401016\\
8	0.000534925876653783\\
9	0.000495138665052242\\
10	0.000461963808922994\\
11	0.000433732240661455\\
12	0.000409115863059639\\
13	0.000387480918167649\\
14	0.000368115820117268\\
15	0.000350779442670687\\
16	0.000335129591017003\\
17	0.000320850827819699\\
18	0.000307765221537512\\
19	0.000295693589017487\\
20	0.00028451447608234\\
21	0.000274208135497833\\
22	0.000264592788987883\\
23	0.000255628472531827\\
24	0.000247206079950865\\
25	0.000239345220455975\\
26	0.000231999475590619\\
27	0.000225043947508677\\
28	0.000218482027159156\\
29	0.000212284181248718\\
30	0.000206409420259138\\
31	0.000200853112546862\\
32	0.000195573973137853\\
33	0.000190585793551204\\
34	0.000185823088469705\\
35	0.000181273679881772\\
36	0.000176937088070581\\
37	0.000172805760680614\\
38	0.000168837883144632\\
39	0.000165036741472682\\
40	0.000161394361989364\\
41	0.000157897085525119\\
42	0.000154541739093392\\
43	0.000151308030102534\\
44	0.000148206589007285\\
45	0.000145211095279124\\
46	0.000142328679856141\\
47	0.000139549742018075\\
48	0.000136863720785141\\
49	0.000134263336105153\\
50	0.000131742388257777\\
51	0.000129307422833444\\
52	0.000126952569079755\\
53	0.000124662973828432\\
54	0.000122445150838313\\
55	0.000120290316082988\\
56	0.000118197729388717\\
57	0.000116164472912822\\
58	0.000114184777909476\\
59	0.000112258009851776\\
60	0.000110382109267532\\
};
\addlegendentry{SCSVRG-k=30};

\end{axis}
\end{tikzpicture}%
}
\caption{MNIST dataset.}
\label{fig:emp1_b}
\end{subfigure}
\begin{subfigure}{0.23 \textwidth}
\resizebox{\linewidth}{!}{
%
%
\begin{tikzpicture}

\begin{axis}[%
width=4.520833in,
height=3.565625in,
at={(0.758333in,0.48125in)},
scale only axis,
separate axis lines,
every outer x axis line/.append style={black},
every x tick label/.append style={font=\color{black}},
xmin=0,
xmax=60,
every outer y axis line/.append style={black},
every y tick label/.append style={font=\color{black}},
ymode=log,
ymin=0.0001,
ymax=0.1,
yminorticks=true,
legend style={legend cell align=left,align=left,draw=black}
]
\addplot [color=red,dashed,line width=2.0pt]
  table[row sep=crcr]{%
1	0.0186296075823393\\
2	0.0143375240400616\\
3	0.0130793327155234\\
4	0.0122809567141488\\
5	0.0116814495712727\\
6	0.0111996934101315\\
7	0.0107957894171639\\
8	0.0104493417426824\\
9	0.0101457061663853\\
10	0.00987588057536054\\
11	0.00963303478152444\\
12	0.00941252860708297\\
13	0.00921036072941633\\
14	0.00902408566864477\\
15	0.00885119651528166\\
16	0.0086899137974486\\
17	0.00853885223989886\\
18	0.00839680253144198\\
19	0.00826284857359222\\
20	0.00813597796708676\\
21	0.00801560221991476\\
22	0.0079011105756745\\
23	0.00779193788799659\\
24	0.00768757395847736\\
25	0.00758769081933175\\
26	0.00749187226528275\\
27	0.00739985599554793\\
28	0.00731131834992654\\
29	0.0072260540161696\\
30	0.00714380437504802\\
31	0.00706439251418955\\
32	0.00698763770856375\\
33	0.00691338830771454\\
34	0.00684146876745384\\
35	0.00677171599358245\\
36	0.00670411018092237\\
37	0.00663838668949934\\
38	0.0065745229018187\\
39	0.00651242336760921\\
40	0.00645197923315705\\
41	0.00639310665325193\\
42	0.00633574242282481\\
43	0.00627982514558745\\
44	0.00622525331100587\\
45	0.00617200203320101\\
46	0.00612003185496107\\
47	0.00606921641794522\\
48	0.00601954013884298\\
49	0.00597097641501598\\
50	0.00592343173620768\\
51	0.00587691999716083\\
52	0.00583135645572602\\
53	0.00578671253265933\\
54	0.0057429799732393\\
55	0.00570008868848843\\
56	0.005658017558201\\
57	0.00561675338924245\\
58	0.00557626099614311\\
59	0.00553650255278137\\
60	0.0054974668605336\\
};
\addlegendentry{SVRG};

\addplot [color=green,solid,line width=2.0pt,mark=asterisk,mark options={solid}]
  table[row sep=crcr]{%
1	0.0213084096609014\\
2	0.00817118204241257\\
3	0.0061678310487524\\
4	0.00528119914178687\\
5	0.00469132875054612\\
6	0.00424015548351686\\
7	0.00387713765610986\\
8	0.0035724446863597\\
9	0.00331222023293032\\
10	0.00308556381941338\\
11	0.00288604117367081\\
12	0.00270815917137313\\
13	0.00254818476765872\\
14	0.0024034782279434\\
15	0.0022717960866282\\
16	0.00215111140467694\\
17	0.00204044839744361\\
18	0.00193835517871899\\
19	0.00184374579112356\\
20	0.0017559329806911\\
21	0.00167408795211588\\
22	0.00159764202816109\\
23	0.00152597350024242\\
24	0.00145880695301431\\
25	0.0013953270917898\\
26	0.00133570284978879\\
27	0.00127919955254635\\
28	0.00122580646134873\\
29	0.00117526998572409\\
30	0.00112730905142339\\
31	0.00108175348596656\\
32	0.00103847918281719\\
33	0.000997255943996023\\
34	0.000957914835818752\\
35	0.000920457831813437\\
36	0.000884642664287449\\
37	0.000850306424375002\\
38	0.000817629250931151\\
39	0.000786134204672773\\
40	0.000756084638948884\\
41	0.000727190935255084\\
42	0.000699555250948503\\
43	0.000672941906951219\\
44	0.000647364460603761\\
45	0.000622797440585854\\
46	0.000599082910328508\\
47	0.00057628093961265\\
48	0.00055426415574461\\
49	0.000533121782449963\\
50	0.000512688456468124\\
51	0.000492976499983033\\
52	0.000473998094681072\\
53	0.000455620004076507\\
54	0.000437912175034638\\
55	0.000420711033516874\\
56	0.000404067715201073\\
57	0.000388033287088141\\
58	0.000372496976079939\\
59	0.000357439454006847\\
60	0.000342842847169622\\
};
\addlegendentry{-k=30};

\end{axis}
\end{tikzpicture}%
}
\caption{CIFAR-10 dataset.}
\label{fig:emp1_e}
\end{subfigure}

\begin{subfigure}{0.23 \textwidth}
\resizebox{\linewidth}{!}{
%
%
\begin{tikzpicture}

\begin{axis}[%
width=4.520833in,
height=3.565625in,
at={(0.758333in,0.48125in)},
scale only axis,
separate axis lines,
every outer x axis line/.append style={black},
every x tick label/.append style={font=\color{black}},
xmin=0,
xmax=60,
every outer y axis line/.append style={black},
every y tick label/.append style={font=\color{black}},
ymode=log,
ymin=0.001,
ymax=1,
yminorticks=true,
legend style={legend cell align=left,align=left,draw=black}
]
\addplot [color=red,dashed,line width=2.0pt]
  table[row sep=crcr]{%
1	0.916574524785411\\
2	0.52840088265196\\
3	0.31493687877865\\
4	0.19594309063101\\
5	0.129624266960602\\
6	0.0901834688659522\\
7	0.0670214604853369\\
8	0.0524727790705476\\
9	0.0428731401489877\\
10	0.0363064208237733\\
11	0.0315922287480749\\
12	0.0280312872210972\\
13	0.0252562783462522\\
14	0.0230188201909587\\
15	0.0211776719782437\\
16	0.0196136544773856\\
17	0.0182653294047779\\
18	0.0170906059658277\\
19	0.0160478543229861\\
20	0.0151155415737028\\
21	0.0142708887417341\\
22	0.0135076410670734\\
23	0.012809991168269\\
24	0.0121667239603847\\
25	0.0115724867217586\\
26	0.0110231336112454\\
27	0.0105120881926729\\
28	0.0100353110364391\\
29	0.00958958852967277\\
30	0.00917084835958165\\
31	0.00877660370679339\\
32	0.00840521964181948\\
33	0.00805475563292249\\
34	0.00772276309639922\\
35	0.00740826185707184\\
36	0.00710878265220812\\
37	0.00682467910651802\\
38	0.00655391338008173\\
39	0.00629543368365928\\
40	0.00604886447236191\\
41	0.00581262751785399\\
42	0.00558663748675809\\
43	0.00537014630936102\\
44	0.00516224137966277\\
45	0.00496258983886475\\
46	0.00477076953153538\\
47	0.00458612861449274\\
48	0.00440836537672725\\
49	0.00423716985676486\\
50	0.00407200529117859\\
51	0.00391265203331662\\
52	0.00375855641211104\\
53	0.00360963024933481\\
54	0.00346562126485886\\
55	0.00332624049934941\\
56	0.00319127679108336\\
57	0.00306063619468092\\
58	0.00293373024824192\\
59	0.00281084137087932\\
60	0.00269169431315995\\
};
\addlegendentry{SVRG};

\addplot [color=green,solid,line width=2.0pt,mark=asterisk,mark options={solid}]
  table[row sep=crcr]{%
1	0.954215213391361\\
2	0.559810288012154\\
3	0.335641563930968\\
4	0.212227226758011\\
5	0.139945860036151\\
6	0.0971839150447275\\
7	0.0705386049100713\\
8	0.0539682814716079\\
9	0.0430331982710944\\
10	0.0355370971854723\\
11	0.0301950074972396\\
12	0.0263406664596724\\
13	0.0233670302738398\\
14	0.0210280872415024\\
15	0.0191439644109428\\
16	0.0175680151447421\\
17	0.0162317444490338\\
18	0.0150766149942043\\
19	0.0140681780100318\\
20	0.0131697095575149\\
21	0.0123681982678498\\
22	0.0116420488754387\\
23	0.0109831826546052\\
24	0.0103820708157317\\
25	0.00982620324763657\\
26	0.00931467894053463\\
27	0.00884007643859063\\
28	0.00839751354538978\\
29	0.00798628058652026\\
30	0.0075997642287052\\
31	0.00723747862233642\\
32	0.00689587090638668\\
33	0.00657402851540741\\
34	0.00626965568655613\\
35	0.00598194725995211\\
36	0.00570870525841678\\
37	0.00544894402647559\\
38	0.00520131267716063\\
39	0.00496570208343939\\
40	0.00474088747643412\\
41	0.00452613453222261\\
42	0.00432062231476952\\
43	0.00412364919233399\\
44	0.00393507808749096\\
45	0.00375392606967153\\
46	0.00357984177470225\\
47	0.00341246697703968\\
48	0.0032516203792612\\
49	0.00309650564721787\\
50	0.00294695738051739\\
51	0.00280262188253753\\
52	0.0026633843347215\\
53	0.00252883655324931\\
54	0.00239871724613417\\
55	0.00227272911838678\\
56	0.00215106426479715\\
57	0.00203304874175475\\
58	0.00191872234757044\\
59	0.0018075926349699\\
60	0.00169993547473131\\
};
\addlegendentry{SCSVRG-k=30};

\end{axis}
\end{tikzpicture}%
}
\caption{RCV1 dataset.}
\label{fig:emp1_c}
\end{subfigure}
\begin{subfigure}{0.23 \textwidth}
\resizebox{\linewidth}{!}{
%
%
\begin{tikzpicture}

\begin{axis}[%
width=4.520833in,
height=3.565625in,
at={(0.758333in,0.48125in)},
scale only axis,
separate axis lines,
every outer x axis line/.append style={black},
every x tick label/.append style={font=\color{black}},
xmin=0,
xmax=60,
every outer y axis line/.append style={black},
every y tick label/.append style={font=\color{black}},
ymode=log,
ymin=1e-05,
ymax=1,
yminorticks=true,
legend style={legend cell align=left,align=left,draw=black}
]
\addplot [color=red,dashed,line width=2.0pt]
  table[row sep=crcr]{%
1	0.112814433361553\\
2	0.0328482177985387\\
3	0.0139481543378353\\
4	0.0079683617733573\\
5	0.00542152993762808\\
6	0.00407587864723688\\
7	0.00324820861089003\\
8	0.00267940622940975\\
9	0.00226673892443235\\
10	0.00195303601335955\\
11	0.00170661159660752\\
12	0.0015091166660732\\
13	0.00134719885444639\\
14	0.00121250404235496\\
15	0.0010989210481031\\
16	0.00100225918565899\\
17	0.000918957011169272\\
18	0.000846685810818738\\
19	0.000783423726223136\\
20	0.000727819753002193\\
21	0.000678479389172156\\
22	0.000634447218292797\\
23	0.000594958960465367\\
24	0.000559424029908745\\
25	0.000527262562608426\\
26	0.000498087034445323\\
27	0.00047157256715992\\
28	0.000447256635423429\\
29	0.000424905746204747\\
30	0.00040434800010939\\
31	0.000385359683378654\\
32	0.000367779189506139\\
33	0.000351493150598035\\
34	0.000336312486226169\\
35	0.00032217474269005\\
36	0.000308991631990906\\
37	0.000296610699918105\\
38	0.000284999695803756\\
39	0.00027407881787931\\
40	0.000263784091611416\\
41	0.000254088662577708\\
42	0.00024493678073273\\
43	0.000236266069896456\\
44	0.000228061566120541\\
45	0.000220262957971555\\
46	0.000212851317313918\\
47	0.000205802467149199\\
48	0.000199084901298663\\
49	0.000192676373306729\\
50	0.000186555635811179\\
51	0.000180707073401225\\
52	0.000175112349616015\\
53	0.000169748050210181\\
54	0.000164605687498778\\
55	0.000159671956467246\\
56	0.000154931114288513\\
57	0.000150373079145678\\
58	0.000145987158373995\\
59	0.000141763023608403\\
60	0.000137688167526344\\
};
\addlegendentry{SVRG};

\addplot [color=green,solid,line width=2.0pt,mark=asterisk,mark options={solid}]
  table[row sep=crcr]{%
1	0.376554740191531\\
2	0.211917317308113\\
3	0.114856601711015\\
4	0.06114728863235\\
5	0.033591690861051\\
6	0.0189533366889875\\
7	0.0110219827371994\\
8	0.00677271984052671\\
9	0.00429727992341172\\
10	0.00287873803856034\\
11	0.00203997731310389\\
12	0.00152165768997911\\
13	0.00118808605625793\\
14	0.000962095321151407\\
15	0.000805521917492801\\
16	0.000691875496288116\\
17	0.000605642438119114\\
18	0.000538755866594264\\
19	0.000484773563603276\\
20	0.000440147326375005\\
21	0.000402451254062984\\
22	0.000370290941613713\\
23	0.000342523577539847\\
24	0.000318143284936408\\
25	0.000296621434004336\\
26	0.0002774689090633\\
27	0.000260247110953314\\
28	0.00024471507256827\\
29	0.000230612557018423\\
30	0.000217756667495808\\
31	0.000205983171229843\\
32	0.000195157252943716\\
33	0.000185167801806711\\
34	0.000175898661049834\\
35	0.000167267586566253\\
36	0.000159229736231455\\
37	0.000151716623777054\\
38	0.000144681087300806\\
39	0.000138058050799333\\
40	0.000131811585392425\\
41	0.000125930671765827\\
42	0.000120358678933383\\
43	0.000115074226975181\\
44	0.000110078555343993\\
45	0.000105315646365657\\
46	0.000100782632005525\\
47	9.64578735320396e-05\\
48	9.23210869614027e-05\\
49	8.83732936383247e-05\\
50	8.45940099721693e-05\\
51	8.0990261745427e-05\\
52	7.75050761221072e-05\\
53	7.41571012642939e-05\\
54	7.09323875244389e-05\\
55	6.78273396314197e-05\\
56	6.48394916643383e-05\\
57	6.19549089342641e-05\\
58	5.91709368311516e-05\\
59	5.64815402122401e-05\\
60	5.38935398711748e-05\\
};
\addlegendentry{SCSVRG-k=30};

\end{axis}
\end{tikzpicture}%
}
\caption{real-sim dataset.}
\label{fig:emp1_d}
\end{subfigure}

\caption{Convergence of Sketched Preconditioned SVRG vs SVRG. The $x$-axis is the number of epochs and the $y$-axis is the suboptimality, $L(\bar{w}_t)-\min_{w \in \reals^d}L(w)$, in logarithmic scale.}
\label{fig:emp1}
\end{figure}
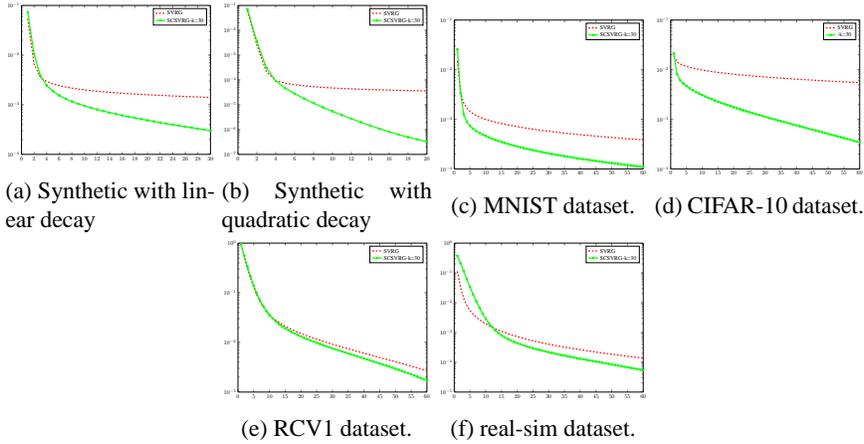

In this section we empirically demonstrate the gain of our method. We consider both regression problems and binary classifications tasks, where the square loss serves as a surrogate for the zero-one loss. We use the following datasets:
\begin{itemize}
\item \emph{Synthetic}: We draw two random $5000 \times 20000$ matrices, $X^{(1)}$ and $X^{(2)}$, whose singular vectors are drawn uniformly at random and the $q$-th singular value is $1/q$ and $1/q^2$, respectively. We then normalize the columns. For each $X=X^{(j)}$, we consider a regression problem, where the labels are generated as follows: we first draw a vector $w^\star \in \cN(0,1)^{5000}$ and then set $y_i = {w^\star}^\top X_{\cdot,i}+z_i$, where $z_i \sim \cN(0,0.1)$. 
\item \emph{MNIST}:\footnote{http://yann.lecun.com/exdb/mnist/} A subset of MNIST, corresponding to the digits $4$ and $7$, where the task is to distinguish between the two digits. Here, $n=12107, d=784$.
\item \emph{RCV1}:\footnote{https://www.csie.ntu.edu.tw/~cjlin/libsvmtools/datasets/} The Reuters RCV1 collection.  Here, $n=20242, d=47236$ and we consider a standard binary document classification task.
\item \emph{CIFAR-10}:\footnote{http://www.cs.toronto.edu/~kriz/cifar.html} Here, $n=50000, d=3072$. Following \cite{frostig2015regularizing}, the classification task is to distinguish between the animal categories to the automotive ones.
\item \emph{real-sim}:\footnote{https://www.csie.ntu.edu.tw/~cjlin/libsvmtools/datasets/} Here, $n=72309, d=20958$, and we consider a standard binary document classification task.
\end{itemize}

\subsection{Inspecting our theoretical speed-up} \label{sec:ratioEmp}
Recall that the ratio \eqref{eq:ratioRidge} quantifies our theoretical speedup.
Hence, we first empirically inspect the prefixes of the corresponding quantities (as a function of $k$) for each of the datasets (see \figref{fig:ratio}). We can see that while in MNIST and CIFAR-10 the ratio is large for small values of $k$, in RCV1 and real-sim the ratio increases very slowly (note that for the former two datasets we use logarithmic scale).

\subsection{Empirical advantage of Sketched Preconditioned SVRG} \label{sec:emp}
We now evaluate \algref{alg:SCSVRG} and compare it to the SVRG algorithm of \cite{xiao2014proximal}. To minimally affect the inherent condition number, we added only a slight amount of regularization, namely, $\lambda=10^{-8}$. The loss used is the square loss. The step size, $\eta$, is optimally tuned for each method. Similarly to previous work on SVRG~\cite{xiao2014proximal, johnson2013accelerating}, the size of each epoch, $m$, is proportional to the number of points, $n$. We minimally preprocessed the data by average normalization: each instance vector is divided by the average $\ell_2$-norm of the instances. The number of epochs is up to $60$. Note that in all cases we choose a small preconditioning parameter, namely $k=30$, so that the preprocessing time of \algref{alg:SCSVRG} is negligible. There is a clear correspondence between the ratios depicted in \figref{fig:ratio} and the actual speedup. In other words, the empirical results strongly affirm our theoretical results. 


\section*{Acknowledgments} 
We thank Edo Liberty for helpful discussions.
The work is supported by ICRI-CI and by the European Research Council (TheoryDL project).

\newpage 

\bibliography{bib}
\bibliographystyle{plain}

\newpage
\appendix

\section{Omitted Proofs}
\begin{proof}  \textbf{(of \thmref{thm:condEffect})}
We first show that the average smoothness of $\tilde{L}$ is bounded by
\begin{equation} \label{eq:condEffectSmooth}
\frac{1}{n+d} \sum_{i=1}^{n+d} \tilde{\beta}_i \le  \tr \left(P^{-1/2} \left(C + \lambda I \right) P^{-1/2} \right)~.
\end{equation}
Note that for any $w$,
\[
\nabla^2 \tilde{\ell}_i(w) = \begin{cases} \frac{n+d}{n}\tx_i \tx_i^\top  & 1 \le i \le n, \\ \lambda (n+d) b_{i-n} b_{i-n}^\top  & n < i \le n+d~.  \end{cases}
\]
Therefore, using the fact that the spectral norm of a rank-$1$ psd matrix is equal to its trace, we obtain
\begin{align*}
\frac{1}{n+d} \sum_{i=1}^n \tilde{\beta}_i &= \frac{1}{n+d} \frac{n+d}{n} \sum_{i=1}^n \|\tx_i \tx_i^\top\| + \frac{1}{n+d}\lambda (n+d) \sum_{j=1}^d \|b_j b_j^\top\| \\
& = \frac{1}{n} \sum_{i=1}^n \tr(\tx_i \tx_i^\top) + \lambda \sum_{j=1}^d \tr(b_i b_i^\top) \\
&= \frac{1}{n}\tr( \sum_{i=1}^n P^{-1/2} x_i x_i^\top P^{-1/2}) + \lambda \, \tr \sum_{j=1}^d (P^{-1/2} e_i e_i^\top P^{-1/2}) \\
&=\tr(P^{-1/2} ( C + \lambda I)P^{-1/2})~.
\end{align*}
Hence, we deduce (\ref{eq:condEffectSmooth}).

We will conclude the theorem by showing that $\tilde{L}$ is $\lambda_d (P^{-1/2} (C + \lambda I)  P^{-1/2})$-strongly convex. Indeed, a similar calculation shows that the Hessian of $L$ at any point $w$ is given by
\begin{align*}
\nabla^2 \tilde{L}(w) &=  P^{-1/2} (C+ \lambda I) P^{-1/2} ~.
\end{align*}
Hence, we conclude the claimed bound.
\end{proof}

\end{document}